\pdfoutput=1
\documentclass[letterpaper]{article} %
\usepackage[preprint]{aaai2027}  %
\usepackage[hyphens]{url}  %
\usepackage{graphicx} %
\usepackage{natbib}  %
\usepackage{caption} %
\usepackage{amsmath,amssymb,amsfonts,amsthm}
\usepackage{booktabs}
\usepackage{array}

\makeatletter
\newcommand{\dataurl}[1]{\begingroup\def\do@url@hyp{}\url{#1}\endgroup}
\newcommand{\nohyphenurlbreaks}{\def\do@url@hyp{}}
\makeatother
\usepackage{tikz}
\usetikzlibrary{arrows.meta,positioning,fit,backgrounds,calc}

\allowdisplaybreaks

\theoremstyle{plain}
\newtheorem{theorem}{Theorem}
\newtheorem{proposition}{Proposition}
\newtheorem{lemma}{Lemma}
\newtheorem{corollary}{Corollary}
\theoremstyle{definition}
\newtheorem{definition}{Definition}
\newtheorem{example}{Example}

\newcommand{\R}{\mathbb{R}}
\newcommand{\E}{\mathbb{E}}
\newcommand{\Prb}{\mathbb{P}}

\newcommand{\calP}{\mathcal{P}}

\newcommand{\calX}{\mathcal{X}}
\newcommand{\calY}{\mathcal{Y}}
\newcommand{\calZ}{\mathcal{Z}}
\newcommand{\calR}{\mathcal{R}}
\newcommand{\ThetaSet}{\Theta}
\newcommand{\MI}{\mathrm{I}}
\newcommand{\KL}{D_{\mathrm{KL}}}
\newcommand{\TV}{\mathrm{TV}}
\newcommand{\Ent}{\mathsf{H}}
\newcommand{\sig}{\sigma}
\newcommand{\one}{\mathbf{1}}
\newcommand{\eps}{\varepsilon}
\newcommand{\ellcyc}{\ell^{\mathrm{cyc}}}
\newcommand{\ellpot}{\ell^{\mathrm{pot}}}
\DeclareMathOperator{\logit}{logit}
\DeclareMathOperator{\diag}{diag}

\DeclareMathOperator{\Var}{Var}
\DeclareMathOperator{\Bern}{Bernoulli}

\DeclareMathOperator*{\argmax}{arg\,max}

\title{Attention-Limited Reward Learning:\\Information-Theoretic Limits of Human Comparisons}

\author{
    Wenqian Xing
}
\affiliations{
    Management Science and Engineering Department\\
    Stanford University\\
    wxing@stanford.edu
}

\begin{document}
\maketitle

\begin{abstract}
Pairwise human comparisons are a primary interface through which modern AI systems learn human preferences, and RLHF pipelines typically model them with Bradley--Terry log-odds governed by latent reward differences. We examine what this assumption misses through a reduced-form model motivated by rational inattention, in which each label is generated by a low-capacity evaluation channel. The model separates two forms of ambiguity that standard reward modeling conflates: a comparison may be difficult because the two candidates are genuinely close in value, or because the relevant distinction is hard to detect under limited attention. We show that limited attention distorts what comparisons reveal in ways more data cannot repair. Passive comparison data cannot generally distinguish reward, attention, and default tendencies, heterogeneous attention can make Bradley--Terry reward modeling recover misleading rankings, and learning is governed not by the number of labels but by the attended information each label carries. A case study on human votes from Chatbot Arena exhibits the predicted signature: a cyclic component of the comparison data that exceeds sampling noise and that no scalar reward can represent. Human feedback should therefore be treated not as direct revealed preference but as an attention-limited measurement process, in which a weak preference signal may reflect hidden evaluation difficulty rather than genuine indifference.
\end{abstract}

\section{Introduction}

In reinforcement learning from human feedback (RLHF), preference-based policy optimization, and related alignment pipelines, the basic measurement primitive is simple: a human is shown two candidate responses, plans, or trajectories and asked which one is better, and a reward model is trained so that its pairwise differences predict these choices. This template underlies much of contemporary reward learning for language models and agentic systems \citep{wirth2017survey,christiano2017deep,ziegler2019finetuning,stiennon2020learning,ouyang2022training}. Direct preference optimization changes the downstream policy update, but still relies on the same comparison signal linking human choices to reward differences \citep{rafailov2023direct}.

The standard statistical abstraction is the Bradley--Terry, or conditional-logit, model \citep{bradley1952rank,luce1959individual,mcfadden1974conditional}. For a query consisting of a context $x$ and two candidates $y^0,y^1$, it assumes
$$
\Prb[y^1 \succ y^0\mid x]
=
\sig\bigl(r(x,y^1)-r(x,y^0)\bigr),
$$
where $\sig(t)=1/(1+e^{-t})$. Comparison labels are then noisy but direct measurements of a single latent reward, and if two candidates are chosen at roughly equal rates the model reads this as near-equal rewards. Much of the statistical ranking literature studies this correctly specified regime \citep{negahban2017rank,shah2017simple}.

For alignment, however, the most important comparisons are often not the easiest ones. A response may be fluent but subtly misleading; a plan may appear helpful while creating long-run risks; an output may satisfy the literal request while violating the user's underlying intent. In such cases the relevant distinction is not necessarily small, it is simply hard to notice. This is the premise of work on scalable oversight, which anticipates that as systems become more capable, the outputs whose evaluation matters most are the ones that strain unaided human judgment \citep{amodei2016concrete,leike2018scalable,irving2018ai,bowman2022measuring}. Evaluators operate with limited time, context, and cognitive bandwidth; they attend to surface features that are cheap to check, such as fluency, length, format, and confidence, and can miss properties requiring costly deliberation, a pattern documented in reward models that absorb such surface regularities \citep{singhal2023length,gao2023scaling}. This is not random noise around a fixed reward. It is a structured measurement problem.

We model this problem using an attention-scaled reduced form motivated by {\em rational inattention}, in which an evaluator allocates scarce information-processing effort before choosing \citep{sims2003implications,matejka2015rational,caplin2015revealed,fosgerau2020discrete}. Some comparisons are cheap to evaluate because one answer is incoherent or plainly violates an instruction; others require careful reasoning about factual dependencies, safety implications, long-horizon consequences, or subtle misalignment. A {\em Shannon rational-inattention first-order condition}, derived in Section~\ref{sec:model}, motivates a marginal comparison channel of the form
\begin{equation}
\ell_z=\eta_z+\beta_z\Delta_z^*,
\label{eq:intro-channel}
\end{equation}
where $\Delta_z^*$ is the {\em deliberative reward difference} for query $z$, $\beta_z\ge0$ is an {\em attention multiplier}, and $\eta_z$ is a default, order, prior, or salience term. Thus the observed label is not just a noisy draw from the reward difference. It is the output of an {\em attention-limited channel} whose strength varies across queries.

This channel highlights a distinction that is central to alignment. A comparison can be \emph{hard because close}, in that the two candidates genuinely have similar deliberative value. A comparison can also be \emph{hard because hidden}, in that the reward-relevant difference is large but the evidence needed to see it is costly or non-salient, so the observed choice probability remains near $50$--$50$. Standard Bradley--Terry modeling treats both cases as small reward gaps. For alignment this conflation is dangerous, because the hard-because-hidden case is where subtle misalignment lives. A weak preference signal need not mean that two outputs are equally good; it may mean that the evaluation protocol failed to elicit the information needed to distinguish them.

\begin{figure*}[t]
\centering
\begin{tikzpicture}[
  font=\footnotesize,
  box/.style={draw=black!55, rounded corners=2.5pt, line width=0.5pt},
  flow/.style={-{Stealth[length=2.4mm]}, line width=0.7pt, draw=black!60},
  edge/.style={-{Stealth[length=2.0mm]}, line width=0.6pt, draw=black!70},
  lbl/.style={font=\scriptsize, text=black!65, align=center},
  ttl/.style={font=\footnotesize\bfseries, text=black!80},
  vert/.style={circle, draw=black!70, fill=white, inner sep=1.1pt,
               font=\scriptsize\bfseries},
  dot/.style={circle, fill=black!75, inner sep=1.4pt}
]
\node[ttl] (t1) at (0.55,2.20) {$R^*$};
\draw[black!40, line width=0.5pt] (0,-0.15) -- (0,1.90);
\node[dot] (rA) at (0,1.60) {};
\node[right=2pt of rA, font=\footnotesize] (rAl) {$A$~($2$)};
\node[dot] (rB) at (0,0.80) {};
\node[right=2pt of rB, font=\footnotesize] (rBl) {$B$~($1$)};
\node[dot] (rC) at (0,0) {};
\node[right=2pt of rC, font=\footnotesize] (rCl) {$C$~($0$)};
\begin{scope}[on background layer]
\node[box, fill=black!3, inner sep=8pt, fit=(t1)(rA)(rC)(rAl)(rCl)] (truth) {};
\end{scope}

\node[ttl] (t2) at (6.50,2.35) {$\ell$};
\node[vert] (vA) at (6.50,1.85) {$A$};
\node[vert] (vB) at (5.35,0.05) {$B$};
\node[vert] (vC) at (7.65,0.05) {$C$};
\draw[edge, line width=1pt] (vA) -- (vB)
  node[midway, lbl, left=2pt] {$\ell=\eps$\\ $\beta=\eps$};
\draw[edge, line width=1pt] (vA) -- (vC)
  node[midway, lbl, right=2pt] {$\ell=\eps$\\ $\beta=\eps/2$};
\draw[edge, line width=1pt] (vB) -- (vC)
  node[midway, lbl, below=2pt] {$\ell=5\eps$,~~$\beta=5\eps$};
\node[lbl] (gnote) at (6.50,-1.05)
  {$\eps+5\eps-\eps\ne0$~~(Prop.~\ref{prop:cycle})};
\begin{scope}[on background layer]
\node[box, fill=blue!4, inner sep=8pt, fit=(t2)(vA)(vB)(vC)(gnote)]
  (obs) {};
\end{scope}

\node[ttl] (t3) at (12.80,2.20) {$r^\dagger$};
\draw[black!40, line width=0.5pt] (12.10,-0.15) -- (12.10,1.90);
\node[dot] (fB) at (12.10,1.60) {};
\node[right=2pt of fB, font=\footnotesize] (fBl) {$B$~($10\eps/3$)};
\node[dot] (fA) at (12.10,1.05) {};
\node[right=2pt of fA, font=\footnotesize] (fAl) {$A$~($8\eps/3$)};
\node[dot] (fC) at (12.10,0) {};
\node[right=2pt of fC, font=\footnotesize] (fCl) {$C$~($0$)};
\draw[{Stealth[length=1.8mm]}-{Stealth[length=1.8mm]}, red!75,
      line width=0.7pt] (14.10,1.05) -- (14.10,1.60)
  node[midway, right=2pt, font=\scriptsize, text=red!75] (revl)
  {$B\succ A$};
\begin{scope}[on background layer]
\node[box, fill=red!4, inner sep=8pt, fit=(t3)(fB)(fC)(fBl)(fAl)(fCl)(revl)]
  (fit) {};
\end{scope}

\coordinate (a1) at (truth.east |- 0,0.95);
\coordinate (a2) at (obs.west  |- 0,0.95);
\coordinate (b1) at (obs.east  |- 0,0.95);
\coordinate (b2) at (fit.west  |- 0,0.95);
\draw[flow] (a1) -- (a2);
\draw[flow] (b1) -- (b2);
\node[lbl] at ($($(a1)!0.5!(a2)$)+(0,0.55)$)
  {$\ell_e=\beta_e\Delta^*_e$};
\node[lbl] at ($($(a1)!0.5!(a2)$)+(0,-0.55)$)
  {heterogeneous $\beta_e$\\
  (attention)};
\node[lbl] at ($($(b1)!0.5!(b2)$)+(0,0.55)$)
  {$(B^\top WB)^{+}B^\top W\ell$};
\node[lbl] at ($($(b1)!0.5!(b2)$)+(0,-0.55)$)
  {(Prop.~\ref{prop:local-projection})};
\end{tikzpicture}
\caption{The projection reversal (Example~\ref{ex:projection-reversal}).
Left to right: deliberative reward $R^*$, observed log odds $\ell$ produced
by the attention channel, and the Bradley--Terry fit $r^\dagger$. Attention
is heterogeneous across pairs, with $\beta_{BC}=5\eps$ but
$\beta_{AC}=\eps/2$. Every observed
majority favors the better item, yet the heterogeneity makes the cycle sum
nonzero, so no scalar reward represents $\ell$
(Proposition~\ref{prop:cycle}); the fit is the projection of $\ell$ onto
potential fields (Proposition~\ref{prop:local-projection}) and ranks $B$
above $A$ (red). 
}
\label{fig:overview}
\end{figure*}
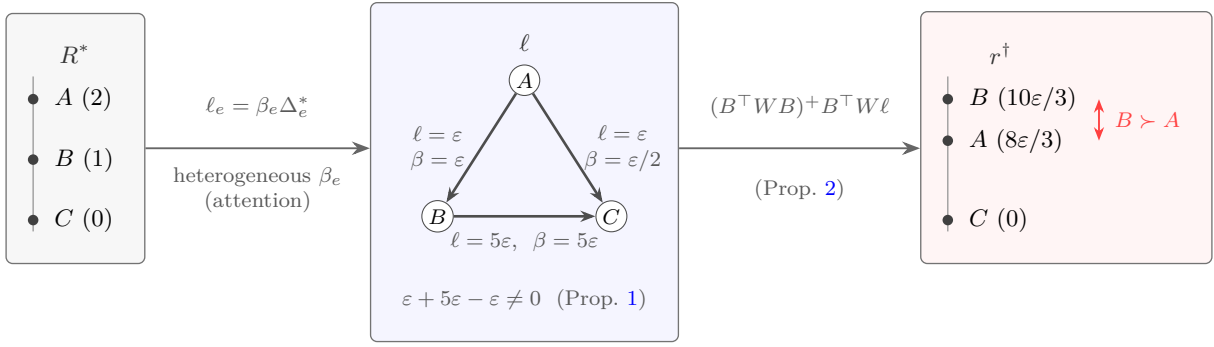

\paragraph{Contributions.}
We make three main contributions.
First, we characterize when human comparisons under limited attention can be represented by a scalar reward.
Section~\ref{sec:model} motivates the comparison channel \eqref{eq:intro-channel} from Shannon rational inattention, and Section~\ref{sec:limits} shows that a Bradley--Terry reward represents the observed log odds if and only if their circulation around every cycle vanishes (Proposition~\ref{prop:cycle}).
While homogeneous attention simply rescales the underlying reward, heterogeneous attention generally violates this condition, so the scalar reward representation assumed in standard reward learning may not exist.

Second, we characterize what a Bradley--Terry learner recovers when such a representation does not exist.
In the low-information regime, the population solution is approximated by a weighted projection of the observed log-odds field onto potential fields (Proposition~\ref{prop:local-projection}).
The recovered reward thus depends on the underlying preferences, the comparison graph, and the query distribution.
Perhaps surprisingly, even when every pairwise majority agrees with the underlying ranking, the Bradley--Terry fit can reverse the ranking of two items (Example~\ref{ex:projection-reversal}, illustrated in Figure~\ref{fig:overview}).
More broadly, reward, attention, and default biases cannot be separately identified from passive comparison probabilities alone (Proposition~\ref{prop:nonidentification}).

Third, we study the fundamental limits of learning from attention-limited comparisons.
A comparison label can have maximal entropy while carrying arbitrarily little information about the underlying reward gap (Proposition~\ref{prop:entropy-information}).
Both mutual and Fisher information about reward scale quadratically with attention $\beta_z$ in the low-attention regime, while the local information matrix has rank one under unknown attention (Proposition~\ref{prop:fisher}).
Distinguishing $\Delta=\delta$ from $\Delta=-\delta$ therefore requires on the order of $1/(\beta^2\delta^2)$ labels.
More generally, reward learning is limited by the total attended information rather than simply the number of annotations (Theorem~\ref{thm:fano}).

We illustrate our theoretical findings through two empirical case studies.\footnote{Code to reproduce both case studies is available at \dataurl{https://github.com/wenqian-xing/attention-limited-reward-learning}.}
Section~\ref{sec:casestudy} examines human votes from Chatbot Arena and finds cyclic patterns in comparison log odds beyond sampling noise.
The resulting information loss under a scalar reward representation is approximated by Proposition~\ref{prop:hodge-kl} to within about ten percent.
Appendix~\ref{app:attention} uses data from a perceptual task to illustrate how response times and gaze can provide information beyond comparison labels.
Appendix~\ref{app:notation} summarizes the notation, and Appendix~\ref{app:proofs} contains the proofs.

\paragraph{Related work.}
Pairwise comparisons are the standard interface for aligning language models and agentic systems \citep{wirth2017survey,christiano2017deep,ziegler2019finetuning,stiennon2020learning,ouyang2022training}, and direct preference optimization retains the Bradley--Terry link between choices and reward differences \citep{rafailov2023direct}. Known failure modes are degradation under optimization pressure \citep{gao2023scaling}, reward hacking \citep{skalse2022defining}, absorption of surface regularities such as length \citep{singhal2023length}, and inherited limits of human oversight \citep{casper2023open}. Our analysis departs from this line by locating the failure in the estimand: with unlimited data and no optimization pressure, the Bradley--Terry target is wrong once evaluator attention is heterogeneous.

Every reward learner embeds a model of the human producing the labels \citep{jeon2020reward}, and one line of work learns systematic departures from Boltzmann rationality \citep{evans2016learning,shah2019feasibility,hadfieldmenell2016cooperative}. Closest to our work is the literature that replaces scalar reward with general preference models and von Neumann winners \citep{azar2024general,munos2024nash,swamy2024minimaximalist}; we complement it by deriving cyclic comparison data from attention alone. Statistical ranking instead assumes Bradley--Terry is correctly specified \citep{negahban2017rank,shah2017simple,zhu2023principled}; we characterize the pseudo-true target under the misspecification attention induces, using the combinatorial Hodge decomposition of \citet{jiang2011statistical}. The behavioral motivation is Shannon rational inattention \citep{sims2003implications,matejka2015rational,caplin2015revealed,fosgerau2020discrete}.

Our bounds give scalable-oversight proposals \citep{amodei2016concrete,leike2018scalable,irving2018ai,bowman2022measuring} a precise target: a protocol helps insofar as it raises the attention multiplier where it matters, and active query selection \citep{sadigh2017active} cannot substitute for attention it leaves unchanged. Appendix~\ref{app:relatedwork} extends the discussion.

\section{A Model of Attention-Limited Human Pairwise Comparisons}
\label{sec:model}

\paragraph{Reward Learning from Pairwise Comparisons.}

Let $x\in\calX$ denote a context and $\calY(x)$ the feasible candidate set. A comparison query is a triple $z=(x,y^0,y^1)\in\calZ$ drawn from a distribution $\rho$ determined by the data-collection policy, and the observed label is $c\in\{0,1\}$, where $c=1$ records a choice of $y^1$ over $y^0$. The target is a deliberative reward $R^*(x,y)\in\R$, the evaluation the principal wants the human to apply after processing all reward-relevant information under the intended criterion; write $\Delta_z^*=R^*(x,y^1)-R^*(x,y^0)$ for the deliberative difference in query $z$. A standard reward learner chooses a score function $r\in\calR$ by logistic maximum likelihood,
\begin{equation*}
\begin{split}
    \widehat r_n\in\argmax_{r\in\calR}
    \frac1n\sum_{t=1}^n
    \bigl\{&
    c_t\log \sig(d_r(z_t))\\
    &+(1-c_t)\log \sig(-d_r(z_t))
    \bigr\},
\end{split}
\end{equation*}
where $d_r(z)=r(x,y^1)-r(x,y^0)$,
under the maintained hypothesis that $\Prb[c=1\mid z]=\sig(d_r(z))$ for some reward aligned with $R^*$. The paper studies what comparison data reveal when human labels are instead attention-filtered measurements of deliberative reward.

\paragraph{Rationally Inattentive Comparisons.}

Fix a query $z$ and let $\omega\in\Omega_z$ denote the evidence relevant to the comparison: facts in the context, latent consequences of the two candidates, safety implications, or other payoff-relevant information. For the rational-inattention benchmark in this subsection, assume $\Omega_z$ is finite and the prior $\mu_z$ has full support. A comparator chooses an attention-decision rule $\pi_z(a\mid\omega)$ over actions $a\in\{0,1\}$, where $a=1$ means choosing $y^1$, with induced marginal $\bar\pi_z(a)=\sum_\omega\mu_z(\omega)\pi_z(a\mid\omega)$. Let $u_z(a,\omega)$ be the comparator's payoff, encoding the evaluation criterion the principal intends. The comparator solves the Shannon rational-inattention problem
\begin{equation}
    \max_{\pi_z}
    \E_{\omega\sim\mu_z,\,a\sim\pi_z(\cdot\mid\omega)}[u_z(a,\omega)]
    -\kappa_z \MI_z(\omega;a),
    \label{eq:ri-problem}
\end{equation}
where $\kappa_z>0$ is the marginal cost of information and $\MI_z(\omega;a)=\sum_{\omega,a}\mu_z(\omega)\pi_z(a\mid\omega)\log\{\pi_z(a\mid\omega)/\bar\pi_z(a)\}$. The next lemma is the standard generalized-logit form of the Shannon solution; it pins down where the attention multiplier and the endogenous action tendency enter.

\begin{lemma}
\label{lem:ri-logit}
At any interior optimum of \eqref{eq:ri-problem},
\begin{equation}
    \pi_z(a\mid\omega)
    =\frac{\bar\pi_z(a)\exp(u_z(a,\omega)/\kappa_z)}
    {\sum_{b\in\{0,1\}}\bar\pi_z(b)\exp(u_z(b,\omega)/\kappa_z)} .
    \label{eq:generalized-logit}
\end{equation}
Consequently,
\begin{equation}
    \log\frac{\pi_z(1\mid\omega)}{\pi_z(0\mid\omega)}
    =\alpha_z+\beta_z\{u_z(1,\omega)-u_z(0,\omega)\},
    \label{eq:conditional-logodds}
\end{equation}
where $\alpha_z=\log\{\bar\pi_z(1)/\bar\pi_z(0)\}$ and $\beta_z=\kappa_z^{-1}$.
\end{lemma}

In \eqref{eq:conditional-logodds} the payoff difference is multiplied by the query-specific inverse information cost $\beta_z$, while the endogenous action tendency $\alpha_z$ enters additively. This is a conditional characterization: it does not by itself pin down the ex-ante marginal label probability obtained by integrating over evidence states, and a deterministic nonzero payoff difference is a boundary case to which the interior logit formula does not apply. The rest of the paper therefore takes the following as a reduced-form marginal measurement channel, motivated by rather than derived from \eqref{eq:conditional-logodds}, with $\eta_z$ standing in for $\alpha_z$ together with exogenous order, format, and salience effects; Appendix~\ref{app:remark} gives one environment in which the channel holds exactly.

\begin{definition}
\label{def:attention-channel}
An attention-scaled comparison channel is a map $z\mapsto q_z\in(0,1)$ satisfying
$$
    q_z=\sig(\ell_z),
    \qquad
    \ell_z=\eta_z+\beta_z\Delta_z^*,
    \qquad
    \beta_z\ge 0.
$$
The case $\eta_z=0$ and $\beta_z\equiv\beta>0$ is homogeneous attention.
\end{definition}

All results below are statements about the reduced-form channel in Definition~\ref{def:attention-channel}. If $\beta_z=0$, the label is independent of the deliberative reward difference given the default: such a query may still produce random-looking labels, but it carries no information about the sign or magnitude of $\Delta_z^*$. A learner observes the choice probabilities $q_z$ while the object of interest is $\Delta_z^*$. Section~\ref{sec:limits} asks what the former reveals about the latter for the standard Bradley--Terry learner; Section~\ref{sec:information} drops the estimator and answers at the level of information, bounding what any procedure could distinguish with any amount of data.

\section{What Bradley--Terry Recovers}
\label{sec:limits}

This section studies the standard Bradley--Terry reward learner under the attention-scaled channel. We first ask whether the observed log odds are compatible with any scalar reward at all; the answer is a cycle criterion that heterogeneous attention generically violates. Practitioners fit Bradley--Terry whether or not the criterion holds, so the next question is what the fit converges to when it fails: a projection of the log-odds field that can reverse the deliberative ranking (Section~\ref{sec:blind-target}). A different parameterization of the reward model does not help, because without further restrictions reward, attention, and defaults are observationally equivalent along a continuum of explanations (Section~\ref{sec:nonidentification}).

\subsection{A Cycle Test for Scalar Representability}

The first question is whether attention-filtered log odds can be represented by any scalar reward at all. This is a purely algebraic question with a complete answer. Fix one context and a finite candidate set $V=\{1,\ldots,m\}$. Let $G=(V,E)$ be the undirected comparison graph. Choose an arbitrary orientation for each edge $e=(i,j)\in E$ and let $q_e=q_{ij}$ be the probability that $i$ is chosen over $j$ in that oriented comparison. Define $\ell_e=\logit(q_e)$, and use the antisymmetric extension $\ell_{ji}=-\ell_{ij}$ when an edge is traversed against its chosen orientation. Let $B\in\R^{E\times V}$ be the signed incidence matrix with $(Br)_e=r_i-r_j$ for $e=(i,j)$. A Bradley--Terry representation of the channel is exactly a potential representation $\ell=Br$. The only obstruction is circulation around cycles.

\begin{proposition}
\label{prop:cycle}
Assume $G$ is connected and $q_e\in(0,1)$ on every compared edge. There exists a score vector $r\in\R^m$ such that $\ell_e=(Br)_e$ on every edge if and only if
$$
    \sum_{k=1}^K \ell_{i_k i_{k+1}}=0,
    \qquad i_{K+1}=i_1,
$$
for every directed cycle $(i_1,i_2,\ldots,i_K,i_1)$ in $G$, where reversing an oriented edge changes the sign of its log odds. When such a score exists, it is unique up to an additive constant.
\end{proposition}

The proposition gives an immediate diagnostic for attention-filtered data. Homogeneous attention is harmless, since it merely rescales the potential; heterogeneous edge-specific attention generally creates nonzero cycle sums.

\begin{corollary}
\label{cor:ri-representable}
On a finite connected comparison graph, attention-scaled log odds
$$
    \ell_{ij}=\eta_{ij}+\beta_{ij}(R_i^*-R_j^*)
$$
admit an exact Bradley--Terry reward if and only if their cycle sums vanish. In particular:
\begin{enumerate}
    \item if $\eta_{ij}=0$ and $\beta_{ij}=\beta>0$ on all edges, then $r_i=\beta R_i^*$ is an exact representation;
    \item if $\beta_{ij}=\beta>0$ and $\eta_{ij}=b_i-b_j$ for some potential $b\in\R^m$, then $r_i=\beta R_i^*+b_i$ is an exact representation.
\end{enumerate}
\end{corollary}

The second case illustrates a subtle failure mode. A default with potential structure, such as a systematic preference for familiar formats or concise answers, is perfectly representable by Bradley--Terry scores. It is nevertheless not deliberative reward; it reverses pairs whenever the default difference dominates the attended reward difference.

\subsection{The Attention-Blind Estimand Is a Projection}
\label{sec:blind-target}

A failed cycle criterion does not stop anyone from fitting Bradley--Terry: the population likelihood remains strictly concave, so the fit converges regardless. The next proposition identifies the limit near the low-information regime as the weighted least-squares projection of the human log-odds field onto the potential fields. The fitted reward is therefore not a noisy copy of deliberative value but a graph-dependent compression of it, shaped by the comparison graph and the sampling policy as much as by human values. Let $W=\diag(\rho_e)$ collect the edge sampling probabilities, with $\rho_e>0$ and $\sum_e\rho_e=1$, normalize scores by $\one^\top r=0$, and for an oriented edge vector $\ell$ define the population objective
\begin{equation}
\begin{split}
    L(r;\ell)
    =\sum_{e\in E}\rho_e
    \bigl\{&\sig(\ell_e)\log\sig((Br)_e)\\
    &+\bigl(1-\sig(\ell_e)\bigr)\log\sig(-(Br)_e)\bigr\}.
\end{split}
    \label{eq:population-bt-graph}
\end{equation}

\begin{proposition}
\label{prop:local-projection}
Suppose $G$ is connected, $\rho_e>0$ on every edge, and $\|\ell\|_\infty\le \eps$. Let $r^\dagger(\ell)$ be the unique maximizer of \eqref{eq:population-bt-graph} on the subspace $\one^\top r=0$. Then, as $\eps\to0$,
\begin{equation}
    r^\dagger(\ell)
    = (B^\top W B)^+B^\top W\ell + O(\eps^3),
    \label{eq:projection-formula}
\end{equation}
where $(\cdot)^+$ denotes the Moore--Penrose inverse acting on the subspace orthogonal to constants. The remainder is in Euclidean norm and is uniform over sufficiently small $\|\ell\|_\infty$.
\end{proposition}

The projection in \eqref{eq:projection-formula} discards all cyclic components of the human log odds, and the induced error is not merely cardinal: it can reverse the learned ranking.

\begin{example}
\label{ex:projection-reversal}
Consider one context with candidates $A,B,C$, deliberative rewards $R_A^*=2$, $R_B^*=1$, $R_C^*=0$, no default bias, and human log odds $\ell_{AB}=\ell_{AC}=\eps$ and $\ell_{BC}=5\eps$ for small $\eps>0$, generated by the attention channel with $\beta_{AB}=\eps$, $\beta_{AC}=\eps/2$, and $\beta_{BC}=5\eps$. Every individual pairwise majority points in the deliberatively correct direction. Sampling the three pairs equally and normalizing $r_C=0$, Proposition~\ref{prop:local-projection} yields $r_A^\dagger=\tfrac83\eps+O(\eps^3)$ and $r_B^\dagger=\tfrac{10}3\eps+O(\eps^3)$, so for small $\eps$ the attention-blind Bradley--Terry target ranks $B$ above $A$ although $A$ is deliberatively best. The reversal is not a knife-edge feature of the expansion; solving the population first-order conditions exactly preserves the same ordering at moderate scales such as $\eps=0.4$.
\end{example}

\subsection{Non-identification of Reward, Attention, and Defaults}
\label{sec:nonidentification}

Example~\ref{ex:projection-reversal} might suggest that the problem lies with the Bradley--Terry functional form, and that a richer model class, a Bayesian prior over rewards, or simply more data would recover the deliberative ranking. The deeper obstacle is that comparison probabilities alone cannot disentangle reward, attention, and defaults, so any method that consumes only passive labels inherits the same ambiguity. The next proposition makes this point in finite-dimensional form.

\begin{proposition}
\label{prop:nonidentification}
Fix comparison probabilities $q_{ij}\in(0,1)$ on a finite set of oriented pairs, and write $\ell_{ij}=\logit(q_{ij})$.
\begin{enumerate}
    \item For any candidate reward vector $v\in\R^m$ and any nonnegative attention multipliers $\beta_{ij}\ge0$, the defaults
    \begin{equation}
        \eta_{ij}=\ell_{ij}-\beta_{ij}(v_i-v_j)
        \label{eq:eta-rationalization}
    \end{equation}
    rationalize the data through the attention-scaled channel.
    \item Suppose defaults are restricted to zero. If, for every compared pair, either $v_i\ne v_j$ and $\ell_{ij}\,(v_i-v_j)\ge0$, or $v_i=v_j$ and $\ell_{ij}=0$, then the same data are rationalized by the pair-specific multipliers
    \begin{equation}
        \beta_{ij}=\frac{\ell_{ij}}{v_i-v_j}
        \label{eq:beta-rationalization}
    \end{equation}
    on all compared pairs with $v_i\ne v_j$, and by any nonnegative $\beta_{ij}$ on zero-gap pairs with $\ell_{ij}=0$.
\end{enumerate}
\end{proposition}

This is the formal reason that a Bayesian prior, a larger neural reward model, or more passive samples cannot on their own recover deliberative reward: without restrictions separating reward from attention and defaults, the likelihood is flat along observationally equivalent explanations. It does not yet say why the flatness arises, or how much data it costs where the parameters are partially informative. Both questions have exact answers once the label is treated as a communication channel, which is the subject of the next section.

\section{Information-Theoretic Limits}
\label{sec:information}

The three failures located in Section~\ref{sec:limits} all concern one particular estimator, so one might hope that a cleverer use of the same labels escapes them. This section shows that they are instead properties of the labels themselves, by re-deriving each failure at the level of information, where no estimator appears.\footnote{``Estimator'' here means a procedure for recovering the latent deliberative reward, not the observed comparison probabilities. A general preference learner, such as Nash learning from human feedback \citep{munos2024nash}, may estimate the observed, possibly cyclic $q_z$ without imposing a scalar reward; our question is whether those observed preferences identify the underlying $\Delta_z^*$.} The weak-signal ambiguity becomes a separation between label entropy and reward information (Section~\ref{sec:entropy-not-information}); non-identification becomes a rank-one information matrix (Section~\ref{sec:fisher}); the estimation failures become sample-complexity lower bounds that bind every procedure, adaptive or not (Section~\ref{sec:kl-lower-bound}); and the cycle obstruction acquires an exact price in KL divergence (Section~\ref{sec:hodge-information}). Defaults are suppressed when they are not essential, since a known default shifts the log odds but does not change the information bottleneck created by $\beta$.

\subsection{Label Entropy Is Not Reward Information}
\label{sec:entropy-not-information}

A nearly even split of labels is often treated as the most informative region of a logistic model because the Bernoulli variance is largest near one half. That intuition is incomplete: if the split is even because attention is near zero, the label may be maximally random while revealing almost nothing about reward. Consider a single pair, let $\Delta$ be an unknown deliberative reward difference with a finite prior, and let the observed label satisfy
\begin{equation}
    C\mid \Delta \sim \Bern\bigl(\sig(\eta+\beta\Delta)\bigr),
    \qquad \beta\ge0,
    \label{eq:random-gap-channel}
\end{equation}
with $\Ent(C)$ the entropy of the marginal label and $\MI(\Delta;C)$ the mutual information between gap and label, in nats.

\begin{proposition}
\label{prop:entropy-information}
In the channel \eqref{eq:random-gap-channel}:
\begin{enumerate}
    \item if $\beta=0$, then $C$ is independent of $\Delta$, so $\MI(\Delta;C)=0$;
    \item if, in addition, $\eta=0$, then $C$ is a fair coin, so $\Ent(C)=\log 2$ while $\MI(\Delta;C)=0$;
    \item if the prior on $\Delta$ has finite support, then as $\beta\to0$,
    \begin{equation}
        \MI(\Delta;C)
        =\frac12\sig(\eta)\bigl(1-\sig(\eta)\bigr)\beta^2\Var(\Delta)+O(\beta^3).
        \label{eq:mi-small-beta}
    \end{equation}
\end{enumerate}
\end{proposition}

Thus a $50$--$50$ outcome has two distinct interpretations. It may mean the reward gap is close to zero, or it may mean the channel is low-transmission because $\beta$ is small. The entropy of the label alone cannot distinguish these cases.

\subsection{Fisher Information}
\label{sec:fisher}

The same distinction appears in local, estimation-theoretic form through Fisher information, where it also exposes the mechanism behind Proposition~\ref{prop:nonidentification}. In standard Bradley--Terry, one label contains the most local information about the reward difference near a $50$--$50$ split; under the attention-scaled channel that maximum is multiplied by $\beta^2$.

\begin{proposition}
\label{prop:fisher}
For one comparison with
$$
    C\sim \Bern(q),
    \qquad
    q=\sig(\eta+\beta\Delta),
$$
the following hold.
\begin{enumerate}
    \item If $\eta$ and $\beta$ are known, the Fisher information in one label about $\Delta$ is
    \begin{equation}
        \mathcal I_\Delta(\Delta)=\beta^2 q(1-q)\le \frac{\beta^2}{4}.
        \label{eq:fisher-known-beta}
    \end{equation}
    \item If $\eta$ is known but $(\Delta,\beta)$ are both unknown, the Fisher information matrix for $(\Delta,\beta)$ is the outer product
    \begin{equation}
        \mathcal I_{(\Delta,\beta)}
        =q(1-q)
        \begin{pmatrix}\beta\\ \Delta\end{pmatrix}
        \begin{pmatrix}\beta & \Delta\end{pmatrix},
        \label{eq:fisher-rank-one}
    \end{equation}
    and therefore has rank at most one.
    \item If $\eta$, $\Delta$, and $\beta$ are all unknown, the same computation gives $\mathcal I_{(\eta,\Delta,\beta)}=q(1-q)vv^\top$ with $v=(1,\beta,\Delta)^\top$, again of rank at most one.
\end{enumerate}
When $\beta$ is constrained to be nonnegative, the usual regular Fisher-information interpretation for parameters involving $\beta$ applies at interior points $\beta>0$; at the boundary $\beta=0$, the displays are local derivative formulas for the logit index.
\end{proposition}

The first claim is a sample-complexity penalty: low attention reduces per-label information about reward quadratically. The rank-one claims are a local version of Proposition~\ref{prop:nonidentification}, since the likelihood identifies the log-odds index $\eta+\beta\Delta$ and not its three components separately.

\subsection{Sample Complexity and Rank Recovery}
\label{sec:kl-lower-bound}

Fisher information measures local difficulty; the operational question is how many labels it takes to settle which candidate is better. Ordinal recovery, too, is bought only with attended information. Consider the clean channel $C\sim\Bern(\sig(\beta\Delta))$ with no default and known attention, and the task of deciding whether $\Delta=\delta$ or $\Delta=-\delta$ for a fixed $\delta>0$.

\begin{proposition}
\label{prop:sign-lower-bound}
Let $P_+$ be the distribution of one label under $\Delta=\delta$ and $P_-$ the distribution under $\Delta=-\delta$ in this channel. If $\beta=0$, then $P_+=P_-$ and no test can have worst-case error probability below $1/2$. If $\beta>0$, then
\begin{equation}
    \KL(P_+\|P_-)
    =\beta\delta\tanh\left(\frac{\beta\delta}{2}\right)
    \le \frac{(\beta\delta)^2}{2}.
    \label{eq:sign-kl}
\end{equation}
For any test based on $n$ independent labels and $\beta>0$, if the worst-case probability of sign error is at most $\alpha<1/2$, then necessarily
\begin{equation}
    n\ge
    \frac{2(1-2\alpha)^2}{\beta\delta\tanh(\beta\delta/2)}
    \ge
    \frac{4(1-2\alpha)^2}{\beta^2\delta^2}.
    \label{eq:sign-lower-bound}
\end{equation}
Consequently, without a positive lower bound on attention, no uniform finite-sample sign guarantee is possible.
\end{proposition}

The bound depends on the product $\beta\delta$. Bradley--Terry interprets a small product as a small reward gap; the attention-scaled model shows it may instead be a large reward gap passing through a low-attention channel. The same logic extends to many reward hypotheses. Let $\ThetaSet$ be a finite set of reward functions, let the learner observe labels $C_1,\ldots,C_n$ under possibly adaptive queries, let $h_{t-1}$ denote the history before label $t$ (previous labels, chosen queries, and any external randomization), and let $P_{\theta,t}(\cdot\mid h_{t-1})$ be the conditional law of label $t$ under hypothesis $\theta$.

\begin{theorem}
\label{thm:fano}
Assume $\theta$ is uniform on a finite class $\ThetaSet$ with $|\ThetaSet|=M\ge2$. Suppose that for every time $t$, every history $h_{t-1}$ with positive probability under the joint mixture distribution induced by the uniform prior over $\ThetaSet$, and every pair $\theta,\theta'\in\ThetaSet$,
\begin{equation}
    \KL\bigl(P_{\theta,t}(\cdot\mid h_{t-1})\,\|\,P_{\theta',t}(\cdot\mid h_{t-1})\bigr)
    \le d_t.
    \label{eq:per-label-kl-bound}
\end{equation}
Then any estimator based on the observed history satisfies
\begin{equation}
    \Prb[\widehat\theta\ne\theta]
    \ge
    1-\frac{\sum_{t=1}^n d_t+\log 2}{\log M}.
    \label{eq:fano-bound}
\end{equation}
\end{theorem}

In attention-scaled logistic channels the constants $d_t$ are of order $\beta_t^2$ for small attended reward separations, so reward recovery is limited by total attended information $\sum_t d_t$ rather than by the number of labels: repeating low-attention comparisons can be much less valuable than collecting fewer high-attention ones.

\subsection{The KL Price of Cyclic Information}
\label{sec:hodge-information}

Returning to the comparison graph closes the loop with Section~\ref{sec:limits}: the cycle obstruction also has an information-theoretic meaning, and the projection of Proposition~\ref{prop:local-projection} is the object that discards it. A scalar reward can represent only potential fields, so the part of the observed log-odds field orthogonal to all potentials is cyclic comparison information, present in the human comparisons but impossible to encode in any scalar reward. The decomposition below is the combinatorial Hodge decomposition of \citet{jiang2011statistical}, applied to the weighted log-odds field of the attention-scaled channel. Write $\|a\|_W^2=a^\top Wa$ for the weighted norm on edge fields, and define
\begin{equation}
    r^H=(B^\top W B)^+B^\top W\ell,
    \ \
    \ellpot=Br^H,
    \ \
    \ellcyc=\ell-\ellpot.
    \label{eq:hodge-decomposition}
\end{equation}

\begin{proposition}
\label{prop:hodge-kl}
Suppose $G$ is connected, $W=\diag(\rho_e)$ has positive diagonal entries, and $\|\ell\|_\infty\le\eps$. Then:
\begin{enumerate}
    \item $\ell=\ellpot+\ellcyc$ is the weighted orthogonal decomposition of the log-odds field into the potential subspace and its orthogonal complement: $B^\top W\ellcyc=0$ and
    \begin{equation}
        \|\ell\|_W^2=\|\ellpot\|_W^2+\|\ellcyc\|_W^2.
        \label{eq:hodge-pythagorean}
    \end{equation}
    \item The best scalar-reward approximation in weighted KL divergence satisfies
    \begin{align*}
        &\inf_{r:\,\one^\top r=0}
        \sum_{e\in E}\rho_e\,
        \KL\bigl(
        \Bern(\sig(\ell_e))\\
        &\quad\big\|\,
        \Bern(\sig((Br)_e))
        \bigr)
        =\frac18\|\ellcyc\|_W^2+O(\eps^3),
    \end{align*}
    where the remainder is uniform over sufficiently small $\|\ell\|_\infty$.
\end{enumerate}
\end{proposition}

The coefficient $r^H$ is computed by the same operator as the pseudo-true target in \eqref{eq:projection-formula}, so to first order the attention-blind Bradley--Terry limit coincides with the potential component of the human log odds, and Proposition~\ref{prop:hodge-kl} prices what Proposition~\ref{prop:local-projection} discards. The cycle obstruction is not merely a consistency condition: the cyclic energy measures the second-order information loss of compressing the human comparison channel to its potential component.

\section{Cyclic Information in Chatbot Arena Votes}
\label{sec:casestudy}
\label{sec:arena}

\begin{figure*}[t]
\centering
\includegraphics[width=0.85\linewidth]{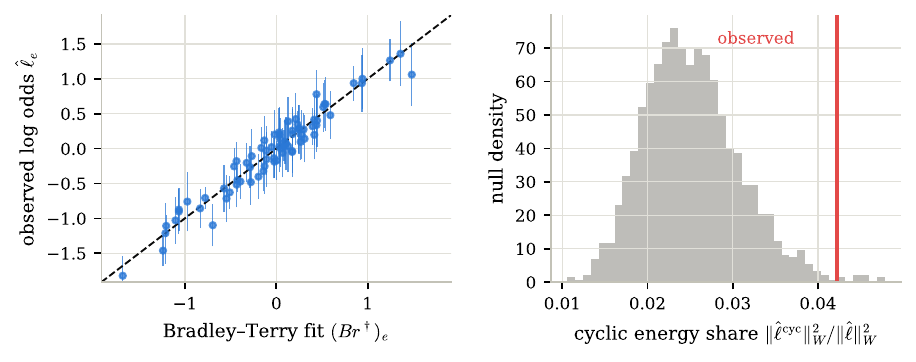}
\caption{Cyclic comparison information in Chatbot Arena votes ($32$ models, $74$ pairs, $15{,}001$ decisive votes). \emph{Left:} observed log odds against the best Bradley--Terry fit, with $\pm2$ standard errors per pair. \emph{Right:} observed cyclic energy share compared with its sampling-noise null ($2{,}000$ parametric bootstrap replicates). The observed share exceeds the null $95$th percentile ($p=0.008$).}
\label{fig:arena}
\end{figure*}

In this section, we examine whether the comparison patterns suggested by our theory arise in real-world human preference data.
Recall from Proposition~\ref{prop:nonidentification} that attention cannot be identified from passive comparison probabilities alone.
Instead, we examine two empirical implications of our theory: heterogeneous attention can introduce cyclic patterns that cannot be captured by any scalar reward model, and measurements of the evaluation process may reveal information about reward differences that labels alone cannot capture.
We study the first implication using Chatbot Arena data in this section and the second through a separate case study in Appendix~\ref{app:attention}.

We consider $57{,}477$ pairwise battles between $64$ language models collected on Chatbot Arena \citep{chiang2024chatbot}.\footnote{Publicly available at \dataurl{https://huggingface.co/datasets/lmarena-ai/arena-human-preference-55k}.}
In our framework, items correspond to models, edges to model pairs, and comparison labels to human votes, with the empirical comparison frequencies determining the weight matrix $W$ in Section~\ref{sec:limits}.
We exclude ties ($31\%$ of all battles), retain pairs with at least $100$ decisive votes, and restrict our analysis to the largest connected component.
This leaves $m=32$ models, $74$ pairs, and $15{,}001$ decisive votes, with a cycle space of dimension $43$.
We estimate the edge log odds using the Haldane--Anscombe estimate $\hat q_e=(w_e+\tfrac12)/(n_e+1)$ of the comparison probability.

Since sampling noise alone can introduce cyclic energy, we conduct a parametric bootstrap test by fitting a Bradley--Terry model through weighted maximum likelihood and regenerating binomial votes at the observed sample sizes $n_e$ over $2{,}000$ replicates.
The observed cyclic energy share is $4.2\%$, compared with a null mean of $2.5\%$ and a null $95$th percentile of $3.5\%$, yielding a bootstrap $p$-value of $0.008$ (Figure~\ref{fig:arena}).
The likelihood-ratio test gives a similar result ($70.6$ on $43$ degrees of freedom, bootstrap $p=0.007$).
We also find that the best scalar fit loses $0.0034$ bits per comparison in weighted KL divergence, compared with the approximation of $0.0038$ bits given by Proposition~\ref{prop:hodge-kl}.
The approximation remains fairly accurate even though the largest observed log odds is close to $1.8$, outside the small-$\eps$ regime considered in our theory.
After accounting for sampling noise using the deviance and its degrees of freedom, the estimated population-level misfit is approximately $0.0013$ bits per comparison.

The excess cyclic energy remains statistically significant under alternative vote thresholds of $50$ and $200$ per pair (bootstrap $p=0.046$ and $0.008$, respectively).
When stratifying by prompt domain, we find the largest estimated excess for coding prompts, where evaluating response quality may require more attention.
However, these results should be interpreted with two caveats: the votes aggregate across prompts and annotators, which can itself introduce cyclic patterns, and our analysis conditions on the model pairs selected by the platform.
Taken together, our findings provide evidence against a scalar reward representation of the aggregate comparison probabilities and suggest that Proposition~\ref{prop:hodge-kl} provides a useful approximation of the resulting information loss.
While consistent with our theory of attention-limited comparisons, these findings do not establish attention limitations as the underlying mechanism.
Appendix~\ref{app:arena} provides further details and robustness checks.

\section{Conclusion}
\label{sec:conclusion}

In this paper, we study how limited human attention affects what can be learned from pairwise comparisons.
Our results show that attention-limited comparisons are not simply noisy measurements of an underlying reward.
When attention varies across queries, comparison log odds may no longer admit a scalar reward representation, and the Bradley--Terry fit can even reverse the underlying preference ranking.
More fundamentally, reward, attention, and default biases cannot be separately identified from comparison probabilities alone.
Our information-theoretic analysis shows that per-label information about reward scales quadratically with attention in the low-attention regime, and that the limits of reward learning depend on total attended information rather than the number of labels collected.
We also characterize the information lost when cyclic comparisons are represented by scalar rewards and illustrate our findings using Chatbot Arena data.

These findings suggest that reward learning is as much a problem of eliciting informative human judgments as it is a problem of statistical estimation.
A nearly even split in comparison labels may reflect either genuine indifference or a failure to recognize important differences.
This distinction is particularly important for AI alignment, where safety-relevant differences may be difficult for human evaluators to detect.
From this perspective, one role of scalable oversight \citep{leike2018scalable,irving2018ai,bowman2022measuring} is to help evaluators attend to differences they might otherwise miss, rather than simply collecting more preference labels.

An important direction for future work is to experimentally study how evaluation time, information access, and AI assistance affect the information contained in human comparisons.
Another direction is to investigate how additional measurements, such as response times, can improve reward identification and guide the allocation of limited annotation resources.

\section*{Acknowledgments}
We thank Ramesh Johari for general advising and support, Junze (Tony) Ye for discussions on rational inattention, and Shixiang (Woody) Zhu for discussions on Nash learning from human feedback. We used AI tools to check and improve the writing of the paper and to find relevant prior work. These tools also assisted in writing the proofs in the appendix.

\bibliography{references}

\appendix

\section{Notation}
\label{app:notation}

Table~\ref{tab:notation} collects the notation of the main text, grouped by the section that introduces it; symbols local to a single proof are omitted. Subscripts $z$, $e$, and $ij$ index queries, oriented edges, and item pairs, and Section~\ref{sec:information} drops the subscript when a statement concerns a single comparison.

\begin{table*}[!t]
\centering
\small
\begin{tabular}[t]{@{}l@{\hspace{6pt}}>{\raggedright\arraybackslash}p{2.2in}@{}}
\toprule
Symbol & Meaning \\
\midrule
\multicolumn{2}{@{}l}{\emph{Queries and rewards (Section~\ref{sec:model})}} \\
$x\in\calX$ & context \\
$\calY(x)$ & feasible candidate set for $x$ \\
$z\in\calZ$ & comparison query $(x,y^0,y^1)$ \\
$\rho$ & query distribution set by the data-collection policy \\
$c\in\{0,1\}$ & observed label; $c=1$ records a choice of $y^1$ \\
$R^*(x,y)$ & deliberative reward: the evaluation the principal intends once all reward-relevant information is processed \\
$\Delta_z^*$ & deliberative difference $R^*(x,y^1)-R^*(x,y^0)$ \\
$r\in\calR$ & score function (reward model) in class $\calR$ \\
$d_r(z)$ & score difference $r(x,y^1)-r(x,y^0)$ \\
$\widehat r_n$ & logistic maximum-likelihood fit to $n$ labels \\
$\sig$, $\logit$ & logistic function $\sig(t)=1/(1+e^{-t})$ and its inverse \\
\addlinespace
\multicolumn{2}{@{}l}{\emph{Rational inattention (Lemma~\ref{lem:ri-logit})}} \\
$\omega\in\Omega_z$, $\mu_z$ & evidence state; its full-support prior \\
$a\in\{0,1\}$ & comparator's action; $a=1$ chooses $y^1$ \\
$\pi_z(a\mid\omega)$ & attention-decision rule \\
$\bar\pi_z(a)$ & marginal action probability \\
$u_z(a,\omega)$ & comparator's payoff, encoding the intended criterion \\
$\kappa_z$ & marginal cost of information \\
$\MI_z(\omega;a)$ & mutual information between evidence and action \\
$\alpha_z$ & endogenous action tendency $\log\{\bar\pi_z(1)/\bar\pi_z(0)\}$ \\
\addlinespace
\multicolumn{2}{@{}l}{\emph{Attention-scaled channel (Definition~\ref{def:attention-channel})}} \\
$q_z$ & probability that $y^1$ is chosen, $q_z=\sig(\ell_z)$ \\
$\ell_z$ & observed log odds $\eta_z+\beta_z\Delta_z^*$ \\
$\beta_z\ge0$ & attention multiplier; $\beta_z=\kappa_z^{-1}$ in Lemma~\ref{lem:ri-logit} \\
$\eta_z$ & default: action tendency plus order, format, and salience effects \\
\bottomrule
\end{tabular}
\hfill
\begin{tabular}[t]{@{}l@{\hspace{6pt}}>{\raggedright\arraybackslash}p{2.2in}@{}}
\toprule
Symbol & Meaning \\
\midrule
\multicolumn{2}{@{}l}{\emph{Comparison graphs (Sections~\ref{sec:limits} and~\ref{sec:hodge-information})}} \\
$V=\{1,\ldots,m\}$ & candidates in one context \\
$G=(V,E)$ & comparison graph; $e=(i,j)\in E$ is an oriented edge \\
$q_e$, $\ell_e$ & probability that $i$ beats $j$ on edge $e$; $\ell_e=\logit(q_e)$, with $\ell_{ji}=-\ell_{ij}$ \\
$R_i^*$, $\eta_{ij}$, $\beta_{ij}$ & reward of item $i$; default and attention on pair $(i,j)$ \\
$B\in\R^{E\times V}$ & signed incidence matrix, $(Br)_e=r_i-r_j$ \\
$\calP$ & subspace of potential fields, $\operatorname{im}(B)$ \\
$\rho_e$, $W$ & edge sampling probabilities; $W=\diag(\rho_e)$ \\
$\|a\|_W$ & weighted norm on edge fields, $\|a\|_W^2=a^\top Wa$ \\
$\one$ & all-ones vector; scores are normalized by $\one^\top r=0$ \\
$(\cdot)^+$ & Moore--Penrose inverse on $\one^\perp$ \\
$L(r;\ell)$ & population Bradley--Terry log-likelihood \eqref{eq:population-bt-graph} \\
$r^\dagger(\ell)$ & its maximizer, the attention-blind target \\
$\eps$ & bound on $\|\ell\|_\infty$ in the local expansions \\
$r^H$, $\ellpot$, $\ellcyc$ & Hodge coefficient; potential and cyclic parts of $\ell$ \eqref{eq:hodge-decomposition} \\
$b$, $v\in\R^m$ & potential default (Corollary~\ref{cor:ri-representable}); candidate reward (Proposition~\ref{prop:nonidentification}) \\
$w_e$, $n_e$, $\hat q_e$ & Chatbot Arena wins and decisive votes on edge $e$; Haldane--Anscombe estimate of $q_e$ \\
\addlinespace
\multicolumn{2}{@{}l}{\emph{Information-theoretic limits (Section~\ref{sec:information})}} \\
$\Delta$, $\beta$, $\eta$ & reward gap, attention, and default of a single pair \\
$C$, $\Bern(p)$ & random label; Bernoulli law with mean $p$ \\
$\Ent(C)$, $\MI(\Delta;C)$ & label entropy; gap--label mutual information (nats) \\
$\mathcal I$ & Fisher information about the subscripted parameters \\
$\KL$, $\TV$ & Kullback--Leibler divergence; total variation distance \\
$\delta$, $P_\pm$ & sign-recovery gap; one-label laws under $\Delta=\pm\delta$ \\
$\alpha$ & worst-case sign-error probability (Proposition~\ref{prop:sign-lower-bound}); unrelated to $\alpha_z$ \\
$n$ & number of labels \\
$\ThetaSet$, $M$, $\theta$ & finite class of reward functions, $M=|\ThetaSet|$; true hypothesis \\
$h_{t-1}$ & history before label $t$ ($H_{t-1}$ when random) \\
$P_{\theta,t}$, $d_t$ & law of label $t$ given the history under $\theta$; per-label KL bound \eqref{eq:per-label-kl-bound} \\
\bottomrule
\end{tabular}
\caption{Notation, grouped by the part of the paper that introduces it. Symbols local to a single proof are omitted.}
\label{tab:notation}
\end{table*}

\section{An Exact Conditional Interpretation}
\label{app:remark}

There is one environment in which the channel holds exactly rather than as a reduced form. Suppose the evidence state $\omega$ is realized but unknown to the comparator ex ante, and let $\Delta_z^*=u_z(1,\omega)-u_z(0,\omega)$ be the payoff difference in the realized state, which is the deliberative difference of Section~\ref{sec:model} when the payoff encodes the intended criterion. If the optimum of \eqref{eq:ri-problem} is interior, Lemma~\ref{lem:ri-logit} implies that, conditional on the realized state, the label distribution is exactly $\sig(\alpha_z+\beta_z\Delta_z^*)$, which is Definition~\ref{def:attention-channel} with $\eta_z=\alpha_z$. In this benchmark, the default is endogenous, determined by the prior $\mu_z$ and the information cost $\kappa_z$, and this is why priors and salience enter additively rather than being scaled like reward. The reduced form frees $\eta_z$ from this benchmark to accommodate order and format effects as well.

\begin{figure*}[!t]
\centering
\includegraphics[width=0.72\linewidth]{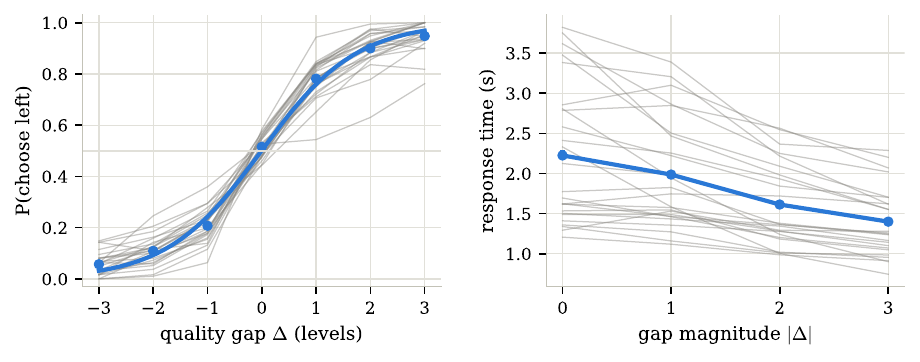}
\caption{Choices and response times in a perceptual comparison task \citep{tavares2017attentional}; $25$ participants, $31{,}854$ trials. \emph{Left:} psychometric curves per participant (gray) and pooled (blue); slopes vary by a factor of six across evaluators. \emph{Right:} mean response time against gap magnitude, per participant and pooled with $95\%$ intervals. The label carries $0.33$ bits about the signed gap $\Delta$ and $0.0001$ bits about its magnitude $|\Delta|$, while response time carries $0.035$ bits about the magnitude.}
\label{fig:attention}
\end{figure*}

\section{Extended Related Work}
\label{app:relatedwork}

\emph{Reward learning from human feedback and its failure modes.}
Pairwise comparisons are the standard interface for aligning language models and agentic systems \citep{wirth2017survey,christiano2017deep,ziegler2019finetuning,stiennon2020learning,ouyang2022training}; direct preference optimization removes the explicit reward model but keeps the Bradley--Terry link between choices and reward differences \citep{rafailov2023direct}. A growing literature documents failure modes of this pipeline. Reward models degrade under optimization pressure \citep{gao2023scaling}, admit reward hacking through gaps between proxy and target \citep{skalse2022defining}, absorb surface regularities such as response length \citep{singhal2023length}, and inherit fundamental limitations of human oversight \citep{casper2023open}. We contribute an attention-based account of one such limitation. Even with unlimited data and no optimization pressure, the Bradley--Terry estimand itself can be wrong once evaluator attention is heterogeneous.

\emph{Human models in reward inference.}
Any reward-learning method embeds a model of the human data generator \citep{jeon2020reward}. The classical choices are noiseless or Boltzmann-rational behavior \citep{luce1959individual}; richer models treat systematic human biases as objects to be learned \citep{evans2016learning,shah2019feasibility} or cast value learning as a cooperative game between human and machine \citep{hadfieldmenell2016cooperative}. The attention-scaled channel is motivated by endogenous attention allocation. Costly evaluation makes the distortion query-specific, correlated with evaluation difficulty, and, as Section~\ref{sec:limits} shows, unidentifiable from passive labels.

\emph{Beyond Bradley--Terry.}
Motivated in part by intransitivity in human preference data, recent work replaces the scalar-reward assumption with general preference models and solves for von Neumann winners or Nash equilibria of a preference game \citep{azar2024general,munos2024nash,swamy2024minimaximalist}. Our results are complementary. The attention-scaled channel shows how cyclic comparison data can arise from heterogeneous attention, and Proposition~\ref{prop:hodge-kl} prices the information that any scalar reward must discard.

\emph{Statistical ranking and rational inattention.}
The statistical ranking literature studies estimation from pairwise comparisons when Bradley--Terry or a parametric relative is correctly specified \citep{negahban2017rank,shah2017simple,zhu2023principled}; we instead characterize the pseudo-true target under the misspecification that attention induces. The potential--cyclic decomposition of Section~\ref{sec:hodge-information} is the combinatorial Hodge decomposition introduced to statistical ranking by \citet{jiang2011statistical}; here the cyclic component arises from heterogeneous attention, and we quantify its KL cost for scalar reward fitting. A behavioral motivation is Shannon rational inattention. The generalized-logit structure of optimal behavior follows \citet{matejka2015rational}, building on \citet{sims2003implications}, with revealed-preference and discrete-choice formulations in \citet{caplin2015revealed} and \citet{fosgerau2020discrete}. We use this theory to motivate, rather than fully identify from first principles, a measurement channel for alignment feedback.

\emph{Scalable oversight.}
The premise that the most alignment-relevant comparisons are the hardest to evaluate motivates scalable-oversight proposals such as recursive reward modeling, debate, and evaluation assistance \citep{amodei2016concrete,leike2018scalable,irving2018ai,bowman2022measuring}. Our lower bounds formalize what such interventions must accomplish. Because sign, ranking, and reward recovery are governed by attended information, an oversight protocol adds value insofar as it raises the attention multiplier on the comparisons that matter, and active query selection \citep{sadigh2017active} cannot substitute for attention it does not change.

\section{Robustness Checks for the Chatbot Arena}
\label{app:arena}

The excess cyclic energy survives vote thresholds of $50$ and $200$ per pair (bootstrap $p=0.046$ and $0.008$). Stratifying battles by prompt domain gives per-domain graphs too sparse for a powered test, but the point estimates order as the attention account predicts: the excess cyclic share is largest on coding prompts, where evaluation is hardest ($8.1\%$ against a null mean of $4.5\%$, bootstrap $p=0.088$), smaller on creative prompts ($4.1\%$ against $2.9\%$), and absent on mathematical prompts. The ordering is suggestive rather than confirmatory. Counting ties as half wins shrinks every log odds toward zero and weakens the excess, to $p=0.06$ on the same graph and to insignificance when the threshold also admits sparsely compared pairs, consistent with dilution by a large mass of uninformative labels. Two caveats bound the interpretation. The votes aggregate over prompts and annotators, so cyclic energy can also arise from aggregation rather than per-query attention, and the analysis conditions on the pairs the platform chose to sample. What the case study establishes is a statistical rejection of every scalar score for this comparison data, at a misfit priced well by the one-eighth law; the finding is consistent with, though not unique evidence for, the attention mechanism. A self-contained script reproduces all numbers.

\section{Case Study: Response Times and Gaze}
\label{app:attention}
\label{sec:attention}

The mechanistic footprint of Section~\ref{sec:casestudy} concerns the channel itself, and testing it needs data in which the deliberative gap and the evaluator's information acquisition are both measured. The perceptual comparison dataset of \citet{tavares2017attentional} provides both.\footnote{Distributed with the aDDM toolbox: \dataurl{https://github.com/goptavares/aDDM-Toolbox}.} Twenty-five participants made $31{,}854$ binary comparisons between two rotated bars, choosing the one closer to a target orientation, with angular distances set by the experimenter; the quality gap $\Delta\in\{-3,\ldots,3\}$ is therefore ground truth rather than a rating proxy, and every trial records the choice, the response time, and gaze fixations. Pooled choices follow a logistic in $\Delta$ with slope $1.14$, and per-participant slopes range from $0.45$ to $2.70$, a factor of six (Figure~\ref{fig:attention}, left). The same objective gap passes through channels whose gain differs by a factor of six across evaluators. The reduced form absorbs attention, effort, and acuity alike into $\beta$, so this is heterogeneous $\beta$ measured directly, whatever its source; and because preference datasets pool annotators over different pairs, evaluator-level heterogeneity of this size induces the edge-level heterogeneity that generates cyclic energy of the kind found in Section~\ref{sec:arena}.

The headline is an information decomposition. Because the design is symmetric, a single label reveals which item is better but almost nothing about by how much; empirically the label carries $0.33$ bits about the signed gap $\Delta$ and $0.0001$ bits about its magnitude $|\Delta|$. Response time carries $0.035$ bits per trial about the magnitude (permutation-corrected, $p<0.002$), while the label carries essentially none; mean response time falls from $2.2$ seconds on ties to $1.4$ seconds at the largest gap (Figure~\ref{fig:attention}, right). Repeated labels do not close this gap, since they reveal only the confounded index $\beta\Delta$ of Proposition~\ref{prop:fisher}, while response time gives a reading from outside the label channel. The quantity an annotation protocol needs in order to separate a near-tie from a large-but-unresolved gap is absent from the label and present in the response time, which is exactly the measurement the theory recommends collecting.

Gaze is associated with an additive shift of the kind the default $\eta$ describes. At fixed $\Delta$, one second of relative dwell toward an item is associated with a $0.91$ increase in its choice log odds (standard error $0.03$), about as much as a full quality level, consistent with the gaze bias documented by \citet{krajbich2010visual}; the association is observational, and gaze may follow an emerging choice as well as shape it. Two further caveats apply. The task is perceptual rather than preferential, and response time is endogenous to difficulty, so these are the channel's ingredients observed in a controlled comparison task rather than in RLHF annotation itself. Annotation platforms typically record decision latency and could release it; the analysis here is what that would enable. 

\section{Technical Proofs}
\label{app:proofs}

\subsection{Proof of Lemma~\ref{lem:ri-logit}}

\begin{proof}
Fix $z$ and suppress the query subscript. Write $u(a,\omega)=u_z(a,\omega)$, $\mu=\mu_z$, $\kappa=\kappa_z$, and $\pi=\pi_z$. Because the optimum is interior, $\pi(a\mid\omega)>0$ and $\bar\pi(a)>0$ for both actions. The mutual information can be written as
\begin{align*}
    \MI(\omega;a)
    =&\sum_{\omega,a}\mu(\omega)\pi(a\mid\omega)\log\pi(a\mid\omega)\\
    &-\sum_a \bar\pi(a)\log\bar\pi(a).
\end{align*}

For each $a$ and $\omega$,
\begin{align*}
    \frac{\partial \MI}{\partial \pi(a\mid\omega)}
    &=\mu(\omega)\{\log\pi(a\mid\omega)+1\}\\
    &\quad-\mu(\omega)\{\log\bar\pi(a)+1\}\\
    &=\mu(\omega)\log\frac{\pi(a\mid\omega)}{\bar\pi(a)}.
\end{align*}
The Lagrangian for the simplex constraints $\sum_a\pi(a\mid\omega)=1$ is
\begin{align*}
    \mathcal{L}(\pi,\lambda)
    =&\sum_{\omega,a}\mu(\omega)\pi(a\mid\omega)u(a,\omega)
    -\kappa\MI(\omega;a)\\
    &+\sum_\omega\lambda_\omega\Bigl(\sum_a\pi(a\mid\omega)-1\Bigr).
\end{align*}
The first-order condition for an interior optimum is
$$
    \mu(\omega)u(a,\omega)
    -\kappa\mu(\omega)\log\frac{\pi(a\mid\omega)}{\bar\pi(a)}
    +\lambda_\omega=0.
$$
Since $\mu(\omega)>0$, this is equivalent to
$$
    \pi(a\mid\omega)=\bar\pi(a)\exp\{u(a,\omega)/\kappa+m(\omega)\},
$$
where $m(\omega)=\lambda_\omega/(\kappa\mu(\omega))$ is independent of $a$. Enforcing $\sum_a\pi(a\mid\omega)=1$ determines the normalizing factor and gives Equation~\eqref{eq:generalized-logit}. Taking the ratio between actions $1$ and $0$ gives Equation~\eqref{eq:conditional-logodds} with $\alpha_z=\log\{\bar\pi_z(1)/\bar\pi_z(0)\}$ and $\beta_z=1/\kappa_z$.
\end{proof}

\subsection{Proof of Proposition~\ref{prop:cycle}}

\begin{proof}
If $\ell=Br$, then the sum of $\ell$ around any directed cycle telescopes:
$$
    \sum_{k=1}^K \ell_{i_k i_{k+1}}
    =\sum_{k=1}^K(r_{i_k}-r_{i_{k+1}})=0.
$$
Conversely, assume all directed cycle sums are zero. Fix a reference vertex $v_0$. For any vertex $i$, choose a path $P_i$ from $i$ to $v_0$ and define $r_i$ as the signed sum of $\ell$ along that path, with an edge contributing $\ell_{ab}$ when traversed in its chosen orientation $(a,b)$ and $-\ell_{ab}$ when traversed in reverse. If two paths from $i$ to $v_0$ produced different sums, traversing one path and then the reverse of the other would give a closed walk with nonzero total circulation. Removing repeated vertices decomposes that closed walk into simple cycles, at least one of which would have nonzero circulation, contradicting the assumption. Hence $r_i$ is well-defined.

For an oriented edge $e=(i,j)$, compare the path from $i$ to $v_0$ with the path that first traverses $e$ from $i$ to $j$ and then follows the chosen path from $j$ to $v_0$. Path independence gives $r_i=\ell_{ij}+r_j$, so $r_i-r_j=\ell_{ij}$. Thus $\ell=Br$. If $r$ and $r'$ both represent $\ell$, then $r_i-r_j=r'_i-r'_j$ on every edge, so $r-r'$ is constant on every connected component. Since $G$ is connected, the difference is a global additive constant.
\end{proof}

\subsection{Proof of Corollary~\ref{cor:ri-representable}}

\begin{proof}
The first sentence follows by substituting $\ell_{ij}=\eta_{ij}+\beta_{ij}(R_i^*-R_j^*)$ into Proposition~\ref{prop:cycle}. If $\eta_{ij}=0$ and $\beta_{ij}=\beta$, then $\ell_{ij}=\beta R_i^*-\beta R_j^*$, so $r_i=\beta R_i^*$ represents the field. If $\eta_{ij}=b_i-b_j$ and $\beta_{ij}=\beta$, then
$$
    \ell_{ij}=(\beta R_i^*+b_i)-(\beta R_j^*+b_j),
$$
so $r_i=\beta R_i^*+b_i$ represents the field.
\end{proof}

\subsection{Proof of Proposition~\ref{prop:local-projection}}

\begin{proof}
The objective in \eqref{eq:population-bt-graph} is continuous and concave in $r$. It is coercive on the normalized subspace $\one^\top r=0$: if $\|r\|\to\infty$ with $\one^\top r=0$, then, because $B$ is injective on $\one^\perp$ in finite dimensions, $\|Br\|_\infty\to\infty$. Since $\sig(\ell_e)\in(0,1)$ for every finite $\ell_e$, the Bernoulli logistic term on any edge whose fitted logit diverges tends to $-\infty$, while the remaining edge terms are nonpositive. Hence $L(r;\ell)\to-\infty$ along diverging normalized sequences, so a maximizer exists. On $\one^\perp$, the objective is strictly concave because $G$ is connected and all $\rho_e$ are positive: if $Br\ne Br'$ on some edge, strict concavity of the Bernoulli log-likelihood is strict along that edge, and if $Br=Br'$ then $r-r'$ is constant and therefore zero on the normalized subspace. Hence the maximizer is unique.

The first-order condition is
\begin{equation}
    B^\top W\{\sig(\ell)-\sig(Br)\}=0.
    \label{eq:graph-foc-proof}
\end{equation}
At $\ell=0$, the unique normalized solution is $r=0$. The derivative of the left-hand side of \eqref{eq:graph-foc-proof} with respect to $r$ at $(\ell,r)=(0,0)$ is $-(1/4)B^\top W B$, which is nonsingular on $\one^\perp$ because $G$ is connected and $W$ has positive diagonal entries. The implicit-function theorem therefore gives a smooth map $r^\dagger(\ell)$ in a neighborhood of zero, with $r^\dagger(0)=0$ and $\|r^\dagger(\ell)\|=O(\|\ell\|)$.

For $|t|$ small, $\sig(t)=1/2+t/4+O(t^3)$; the quadratic term is absent because $\sig(t)-1/2$ is odd. Since $\|\ell\|_\infty\le\eps$ and $\|Br^\dagger(\ell)\|_\infty=O(\eps)$, substituting this expansion into \eqref{eq:graph-foc-proof} yields
$$
    B^\top W\left\{\frac{\ell-Br^\dagger(\ell)}{4}\right\}=O(\eps^3).
$$
Multiplying by $4$ and solving on $\one^\perp$ gives
$$
    r^\dagger(\ell)=(B^\top W B)^+B^\top W\ell+O(\eps^3),
$$
with a uniform remainder on a sufficiently small neighborhood of zero for the fixed finite graph and fixed positive weights.
\end{proof}

\subsection{Derivation for Example~\ref{ex:projection-reversal}}

\begin{proof}
Let $r_A=x$, $r_B=y$, and $r_C=0$. With equal edge weights, the Bradley--Terry fitted probabilities are
$$
    p_{AB}=\sig(x-y),\qquad p_{AC}=\sig(x),\qquad p_{BC}=\sig(y).
$$
The true probabilities corresponding to the log odds of Example~\ref{ex:projection-reversal} are
$$
    q_{AB}=\sig(\eps),\qquad q_{AC}=\sig(\eps),\qquad q_{BC}=\sig(5\eps).
$$
The first-order conditions for $x$ and $y$ are
\begin{align*}
    (q_{AB}-p_{AB})+(q_{AC}-p_{AC})&=0,\\
    -(q_{AB}-p_{AB})+(q_{BC}-p_{BC})&=0.
\end{align*}
Using $\sig(t)=1/2+t/4+O(t^3)$ and the smoothness of the optimum from Proposition~\ref{prop:local-projection}, the first-order terms satisfy
$$
    2x-y=2\eps,
    \qquad
    2y-x=4\eps.
$$
Solving gives $x=8\eps/3$ and $y=10\eps/3$. The omitted terms are $O(\eps^3)$ by the same Taylor expansion and implicit-function argument used in Proposition~\ref{prop:local-projection}.
\end{proof}

\subsection{Proof of Proposition~\ref{prop:nonidentification}}

\begin{proof}
For arbitrary $v$ and nonnegative multipliers $\beta_{ij}$, define $\eta_{ij}$ by \eqref{eq:eta-rationalization}. Then
$$
    \sig\{\eta_{ij}+\beta_{ij}(v_i-v_j)\}
    =\sig(\ell_{ij})=q_{ij},
$$
which proves the first claim. For the second claim, first consider a compared pair with $v_i\ne v_j$. The condition $\ell_{ij}\,(v_i-v_j)\ge0$ makes \eqref{eq:beta-rationalization} nonnegative, and it satisfies
$$
    \sig\{\beta_{ij}(v_i-v_j)\}=\sig(\ell_{ij})=q_{ij}.
$$
If $v_i=v_j$, the stated condition gives $\ell_{ij}=0$, so setting, for example, $\beta_{ij}=0$ yields $\sig\{\beta_{ij}(v_i-v_j)\}=\sig(0)=q_{ij}$. Thus all compared pairs are rationalized with zero defaults.
\end{proof}

\subsection{Proof of Proposition~\ref{prop:entropy-information}}

\begin{proof}
If $\beta=0$, then $\Prb[C=1\mid\Delta]=\sig(\eta)$ for every value of $\Delta$, so the conditional distribution of $C$ is independent of $\Delta$ and $\MI(\Delta;C)=0$. If also $\eta=0$, then $\sig(\eta)=1/2$, so $C$ is a fair coin and $\Ent(C)=\log2$.

For the expansion, let $q_0=\sig(\eta)$ and $s=q_0(1-q_0)$. Because the prior has finite support, all Taylor remainders below are uniform over the support of $\Delta$. Write $q_\beta(\Delta)=\sig(\eta+\beta\Delta)$. Then
\begin{align*}
    q_\beta(\Delta)&=q_0+s\beta\Delta+O(\beta^2),\\
    \bar q_\beta:=\E[q_\beta(\Delta)]&=q_0+s\beta\E[\Delta]+O(\beta^2).
\end{align*}
The mutual information is
$$
    \MI(\Delta;C)
    =\E\left[\KL\bigl(\Bern(q_\beta(\Delta))\,\|\,\Bern(\bar q_\beta)\bigr)\right].
$$
For Bernoulli parameters $u$ and $v$ in a compact subinterval of $(0,1)$,
\begin{align*}
    &\KL(\Bern(u)\|\Bern(v))\\
    &\qquad=\frac{(u-v)^2}{2q_0(1-q_0)}+O(\epsilon_{uv}),
\end{align*}
with $\epsilon_{uv}=|u-v|^3+|v-q_0||u-v|^2$. Here
$$
    q_\beta(\Delta)-\bar q_\beta=s\beta(\Delta-\E\Delta)+O(\beta^2),
$$
so
\begin{align*}
    \MI(\Delta;C)
    &=\frac{1}{2s}\E\bigl[s^2\beta^2(\Delta-\E\Delta)^2\bigr]+O(\beta^3)\\
    &=\frac12 s\beta^2\Var(\Delta)+O(\beta^3),
\end{align*}
which is \eqref{eq:mi-small-beta}.
\end{proof}

\subsection{Proof of Proposition~\ref{prop:fisher}}

\begin{proof}
For a Bernoulli observation with parameter $q=\sig(s)$ and scalar index $s$, the score with respect to $s$ is $C-q$ and the Fisher information for $s$ is $q(1-q)$. By the chain rule, for any parameter vector $\theta$ entering only through $s(\theta)$,
$$
    \mathcal I_\theta=q(1-q)\,\nabla s(\theta)\nabla s(\theta)^\top.
$$
If $s=\eta+\beta\Delta$ and $\eta,\beta$ are known, then $\partial s/\partial\Delta=\beta$, giving \eqref{eq:fisher-known-beta}. Since $q(1-q)\le1/4$, the upper bound follows. If $\eta$ is known and $\theta=(\Delta,\beta)$, then $\nabla s=(\beta,\Delta)^\top$, giving \eqref{eq:fisher-rank-one}. If $\theta=(\eta,\Delta,\beta)$, then $\nabla s=(1,\beta,\Delta)^\top$, giving the third claim. Each displayed matrix is an outer product times the positive scalar $q(1-q)$, so its rank is at most one. If the model constrains $\beta\ge0$, the regular Fisher-information interpretation for the parameters involving $\beta$ is at interior values $\beta>0$; the same displayed derivatives remain the local logit-index derivatives at the boundary.
\end{proof}

\subsection{Proof of Proposition~\ref{prop:sign-lower-bound}}

\begin{proof}
If $\beta=0$, then both hypotheses generate $\Bern(1/2)$ labels, so $P_+=P_-$ and every test has worst-case error probability at least $1/2$. Now assume $\beta>0$. Let $a=\beta\delta$ and $q=\sig(a)$. Under $\Delta=\delta$, one label is $\Bern(q)$; under $\Delta=-\delta$, one label is $\Bern(1-q)$. Thus
\begin{align*}
    \KL(P_+\|P_-)
    &=q\log\frac{q}{1-q}+(1-q)\log\frac{1-q}{q} \\
    &=(2q-1)\log\frac{q}{1-q}
      =a\tanh(a/2).
\end{align*}
Because $\tanh(a/2)\le a/2$ for $a\ge0$, the upper bound in \eqref{eq:sign-kl} follows.

For $n$ independent labels, product additivity gives
$$
    \KL(P_+^{\otimes n}\|P_-^{\otimes n})=n\KL(P_+\|P_-).
$$
Let $\widehat s\in\{+,-\}$ be any test and let $\alpha$ bound its worst-case error probability. The total variation distance between the two product distributions must satisfy
$$
    \TV(P_+^{\otimes n},P_-^{\otimes n})\ge 1-2\alpha,
$$
because the optimal testing error equals $(1-\TV)/2$ and no test can outperform the optimal test. Pinsker's inequality gives
$$
    \TV(P_+^{\otimes n},P_-^{\otimes n})
    \le \sqrt{\frac12 n\KL(P_+\|P_-)}.
$$
Combining the two displays yields
$$
    n\KL(P_+\|P_-)
    \ge 2(1-2\alpha)^2,
$$
which gives the first lower bound in \eqref{eq:sign-lower-bound}. Substituting $\KL(P_+\|P_-)\le\beta^2\delta^2/2$ gives the second. If $\beta$ can be arbitrarily close to zero, the necessary sample size can be made arbitrarily large.
\end{proof}

\subsection{Proof of Theorem~\ref{thm:fano}}

\begin{proof}
Let $H_{t-1}$ denote the history before label $t$, including past labels, chosen queries, and any external randomization used by the adaptive query rule. Query choices and external randomization add no information about $\theta$ except through past labels because the randomization is independent of $\theta$ conditional on the past. Thus the chain rule for mutual information gives
$$
    \MI(\theta;H_n)
    =\sum_{t=1}^n \MI(\theta;C_t\mid H_{t-1}).
$$
For any fixed history $h_{t-1}$, the conditional mutual information between a discrete parameter and one observation is bounded by the average pairwise KL divergence among the conditional observation laws, and hence by their maximum. Assumption~\eqref{eq:per-label-kl-bound} therefore implies
$$
    \MI(\theta;C_t\mid H_{t-1}=h_{t-1})\le d_t
$$
for every history with positive probability under the joint mixture distribution. Taking expectations over histories and summing gives
$$
    \MI(\theta;H_n)\le \sum_{t=1}^n d_t.
$$
Fano's inequality for a uniform parameter on $M$ hypotheses states that any estimator based on the observed history satisfies
$$
    \Prb[\widehat\theta\ne\theta]
    \ge 1-\frac{\MI(\theta;H_n)+\log2}{\log M}.
$$
Combining the two displays proves \eqref{eq:fano-bound}.
\end{proof}

\subsection{Proof of Proposition~\ref{prop:hodge-kl}}

\begin{proof}
The vector $r^H$ in \eqref{eq:hodge-decomposition} is the weighted least-squares projection coefficient of $\ell$ onto $\operatorname{im}(B)$ under the normalization orthogonal to constants. The normal equations are
$$
    B^\top W(\ell-Br^H)=0,
$$
which is $B^\top W\ellcyc=0$. Hence $\ellpot\in\calP$, where $\calP=\operatorname{im}(B)$ is the subspace of potential fields, and $\ellcyc\in\calP^{\perp_W}$, its $W$-orthogonal complement, proving the weighted orthogonal decomposition and the Pythagorean identity \eqref{eq:hodge-pythagorean}.

For the KL statement, first note that $r^H=O(\eps)$ because it is a fixed linear map applied to $\ell$. The weighted KL objective differs from $-L(r;\ell)$ in \eqref{eq:population-bt-graph} only by the constant
$$
    \sum_e\rho_e\{\sig(\ell_e)\log\sig(\ell_e)+(1-\sig(\ell_e))\log\sig(-\ell_e)\},
$$
which is independent of $r$, so its centered minimizer coincides with the centered maximizer in Proposition~\ref{prop:local-projection}, which localizes the minimizer in an $O(\eps)$ neighborhood of zero. For $|u|,|v|\le C\eps$, a Taylor expansion of Bernoulli KL around $(u,v)=(0,0)$ in log-odds coordinates gives
\begin{equation}
\begin{split}
    &\KL\bigl(\Bern(\sig(u))\,\|\,\Bern(\sig(v))\bigr)\\
    &\qquad=\frac18(u-v)^2+O(\eps^3),
\end{split}
    \label{eq:bernoulli-kl-logit-expansion}
\end{equation}
with a uniform remainder. Therefore
\begin{align*}
    &\inf_{r:\,\one^\top r=0}
    \sum_e \rho_e
    \KL\bigl(\Bern(\sig(\ell_e))\\
    &\qquad\qquad\|\,\Bern(\sig((Br)_e))\bigr)\\
    &\qquad=\frac18\inf_{r:\,\one^\top r=0}\|\ell-Br\|_W^2+O(\eps^3)\\
    &\qquad=\frac18\|\ellcyc\|_W^2+O(\eps^3),
\end{align*}
where the last equality is the definition of the weighted projection residual.
\end{proof}

\end{document}